%% file: main.tex
\documentclass[10pt,journal,compsoc]{IEEEtran}

\ifCLASSOPTIONcompsoc
  \usepackage[nocompress]{cite}
\else
  \usepackage{cite}
\fi

\ifCLASSINFOpdf
\else
\fi

\input{preamble.tex}

\input{commands.tex}

\begin{document}

\title{A Unified Approach to Coreset Learning}

\author{Alaa Maalouf$^{*}$,
        Gilad Eini$^{*}$,
        Ben Mussay,
        Dan Feldman,
        and~Margarita Osadchy
\IEEEcompsocitemizethanks{\IEEEcompsocthanksitem A. Maalouf and G. Eini and B. Mussay and D. Feldman and M. Osadchy are with the Department of Computer Science, University of Haifa, Israel.
\IEEEcompsocthanksitem Corresponding author: A. Maalouf. E-mail: alaamalouf12@gmail.com.

\IEEEcompsocthanksitem $^{*}$ Equal contribution.}
}

\markboth{Unified Approach for Coreset Learning
}%
{Shell \MakeLowercase{\textit{et al.}}: Bare Demo of IEEEtran.cls for Computer Society Journals}

\IEEEtitleabstractindextext{%
\begin{abstract}
Coreset of a given  dataset and loss function is usually a small weighed set that approximates this loss for every query from a given set of queries.   Coresets have shown to be very useful in many applications. However, coresets construction is done in a problem dependent manner and it could take years to design and prove the correctness of a coreset for a specific family of queries. This could limit coresets use in practical applications. Moreover, small coresets provably do not exist for many problems. 

To address these limitations, we propose a generic, learning-based algorithm for construction of coresets. 
Our approach offers a new definition of coreset, which is a natural relaxation of the standard definition and aims at approximating the \emph{average} loss of the original data over the queries. This allows us to use a learning paradigm to compute a small coreset of a given set of inputs with respect to a given loss function using a training set of queries. We derive formal guarantees for the proposed approach.  Experimental evaluation on deep networks and classic machine learning problems show that our learned coresets yield  comparable or even better results than the existing algorithms with worst-case theoretical  guarantees (that may be too pessimistic in practice). Furthermore, our approach applied to deep network pruning provides the first coreset for a full deep network, i.e., compresses all the network at once, and not layer by layer or similar divide-and-conquer methods.
\end{abstract}

}

\maketitle

\IEEEdisplaynontitleabstractindextext

%
\IEEEpeerreviewmaketitle

\IEEEraisesectionheading{\section{Introduction}\label{sec:introduction}}

\IEEEPARstart{C}{oreset} is usually defined as a small weighted subset of the original input set that provably approximates the given loss (objective) function for every query in a given set of queries.
Coresets are useful in machine learning applications as they offer significant efficiency improvements. Namely, traditional (possibly inefficient, but provably optimal) algorithms can be applied on coresets to obtain an approximation of the optimal solution on the full dataset using time and memory that are smaller by order of magnitudes. Moreover, existing heuristics that already run fast can be improved in terms of accuracy by running them many times on the coreset in the time it takes for a single run on the original (big) dataset. Finally, coresets can be maintained for a distributed \& streaming data, where the stream is distributed in parallel from a server to $m$ machines (e.g. cloud), and the goal is to maintain the optimal solution (or an approximation to it) for the whole input seen so far in the stream using small update time, memory, and communication to the server.

In the recent decade, coresets, under different formal definitions, were applied to many machine learning algorithms
e.g. logistic regression~\cite{huggins2016coresets,munteanu2018coresets}, SVM~\cite{har2007maximum,tsang2006generalized,tsang2005core,tsang2005very,tukan2020coresets}, clustering problems~\cite{ feldman2011scalable,gu2012coreset,jubran2020sets,lucic2015strong, schmidt2019fair}, matrix approximation~\cite{feldman2013turning, maalouf2019fast,maalouf2020tight, sarlos2006improved,maalouf2021coresets}, $\ell_z$-regression~\cite{cohen2015lp, dasgupta2009sampling, sohler2011subspace}, decision trees~\cite{jubran2021coresets}, and others; see surveys~\cite{feldman2020core,phillips2016coresets,jubran2019introduction}.

Some attempts of using coresets were recently suggested in application to deep networks.
Apart from the standard  use of coresets for reducing  the amount of computations in training, e.g., by replacing full data~\cite{mirzasoleiman-icml-2020} or a batch~\cite{batchSizeReduction} with a coreset, there
are other applications that motivate the use of summarization methods in deep networks, e.g., model compression, continual learning, domain adaptation, federated learning, neural architecture search. We discuss some of the them below.

\noindent\textbf{Model Compression.} Deep networks are highly over-parametrized,  resulting in high memory requirements and slower inference. While many methods have been developed for reducing the size of a previously trained network with no (or small) accuracy loss~\cite{liu2019metapruning,li2019learning,chen2020storage,he2019filter,dong2017more,kang2020operation,ye2020good, ye2018rethinking,maalouf2021deep}, most of them relied on heuristics, which performed well on known benchmarks, but diverged considerably from the behavior of the original network on specific sub-sets of input distribution~\cite{brain_workshop}. Few previous works~\cite{MussayOBZF20,Mussai20a,baykal2018datadependent,Liebenwein2020Provable} tried to resolve this problem by deriving a coreset for a fully connected or a convolutional layer with provable trade-off between the compression rate and the approximation error for any future input. 
However, since these works construct a coreset for a layer, the full network compression is performed in a layer-by-layer fashion. 

\noindent\textbf{Limited Data Access.} Problems, such as continual / incremental learning~\cite{LwF,iCarl,Few_shot_reminder,Prototype_reminder,Lopez-PazR17,BorsosM020}, domain adaptation~\cite{LiSWZG17,Asami}, federated learning~\cite{goetz2020federated} do not have access to the full data (due to memory limitation or privacy issues) and only a small summary of it can be used. Coresets offer a natural solution for these problems. 

\noindent\textbf{NAS.} Another important application that could benefit from coresets is neural architecture search (NAS). Evaluating different architectures or a choice of parameters using a large data set is extremely time consuming. A representative, small summary of the training set could be used for a reliable approximation of the full training, while greatly speeding up the search. Few recent works~\cite{Condensation, GTN} (inspired by the work of~\cite{wang2018dataset}) tried to learn a small synthetic set that summarizes the full set for NAS.

Previous attempts of summarizing a full training set with a small subset or a synthetic set showed a merit of using coresets in modern AI (e.g., for training deep network). However, the summarization methods that they suggested were based on heuristics with no guarantees on the approximation error. Hence, it is not clear that existing heuristics for data summarization 
could scale up to real-life problems.
On the other hand, theoretical coresets that provably quantify the trade-off between data reduction and information loss for an objective function of interest, are mostly limited to simple, shallow models due to the challenges discussed below in Section~\ref{subsection:challanges}. 
From the theoretical perspective, it seems that we cannot have coresets for a reasonable neural network under the classic definition of the worst case query (e.g. see Theorem 6~\cite{MussayOBZF20}). 
In this paper we try to find a midway between these two paradigms.

\subsection{Coreset challenges}\label{subsection:challanges} In many modern machine learning problems, obtaining non-trivial theoretical worst-case guarantees is usually impossible (due to a high complexity of the target model, e.g. deep networks or since every point in the input set is important in the sense of high sensitivity~\cite{tukan2020coresets2}). 
Even for the simple problems, it may take years to design a coreset and prove its correctness for a specific problem at hand.

Another problem with the existing theoretical frameworks is the lack of generality. Even the most generic frameworks among them~\cite{feldman2011unified, langberg2010universal} replace the problem of computing a coreset for $n$ points with $n$ new optimization problems (known as sensitivity bound), one for each of the $n$ input points. Solving these, however, might be harder than solving the original problem. Hence, different approximation techniques are usually tailored for each and every problem.

\subsection{Our Contribution}
The above observations suggest that there is a need in a more generic approach that can 
compute a coreset automatically for a given pair of dataset and loss function, and can be applied to hard problems, such as deep networks. It seems that this would require some relaxation in the standard coreset definition. Would such a relaxed coreset produced by a generic algorithm yield comparable empirical results with the traditional coresets that have provable guarantees?
We affirmably answer this question by providing:
\begin{enumerate}[(i)]
    \item A new definition of a coreset, which is a relaxation of the traditional definition of the strong coreset.
    \item AutoCL: a generic and simple algorithm that is designed to compute a coreset (under the new definition) for almost any given input dataset and loss function.
    \item Example applications with highly competitive empirical results for: (a) problems with known coreset construction algorithms, namely, linear regression and logistic regression, where the goal is to summarize the input training set, and (b) model compression, i.e., learning a coreset 
    of all training parameters of a deep neural network at once (useful for model pruning). To our knowledge, this is the first algorithm that aims to compute a coreset for the network at once, and not layer by layer or similar divide-and-conquer methods. It is also the first approach that suggests to represent the coreset itself as a small (trainable) network that keeps improving on each iteration. In this sense we suggest "coreset for deep learning using deep learning".
    \item Open code for reproducing our results~\cite{opencode}. We expect that it would be the baseline for producing ``empirical" coresets for many problems in the future. Mainly, since it requires very little familiarity with the existing theoretical  research on coresets.
   
\end{enumerate}


\section{Preliminaries}

\textbf{Notations.} 
For a set $P$ of $n$ items, we use $\abs{P}$ to denote the number of items in $P$ (i.e., $\abs{P}=n$). For an event $B$ we use $\Pr(B)$ as the probability that event $B$ occurs, and for a random variable $x$ with a probability measure $\mu$, we use $\mathbb{E}_{\mu}(x)$ to denote its mean (expected value). Finally, for a loss function $\mathrm{loss}$ and an input set of variables $C$ (from any form), we use $\nabla \mathrm{loss}(C)$ to denote a standard gradient computation of $\mathrm{loss}$ with respect to the set of variables $C$, and $C- \alpha\nabla \mathrm{loss}(C)$ to denote a standard variables update ($C$) using a gradient step, where $\alpha>0$ is the learning rate.

The following (generic) definition of a query space encapsulates all the ingredients required to formally define an optimization problem.

\begin{definition} [Query space; see Definition 4.2 in~\cite{braverman2016new}] \label{def:querySpace}
Let $\mathbb{P}$ be a (possibly infinite) set called \emph{ground set}, $Q'$ be a (possibly infinite) set called \emph{query set},  and let $f:\mathbb{P}\times Q' \to [0,\infty)$ be a loss (or cost) function. Let $P\subseteq \mathbb{P}$ be a finite set called \emph{input set}, and let $w:P\to [0,\infty)$ be a \emph{weight function}.
The tuple $(P,w,Q',f)$ is called a \emph{query space} over $\mathbb{P}$.
\end{definition}

Typically, in the training step (of machine learning model), we solve the optimization problem, i.e., we aim at finding the solution $q^*$ that minimizes the sum of fitting errors $\sum_{p\in P} w(p) f(p,q)$ over every $q\in Q'$. 

\begin{definition} [Query cost] \label{def:onequerycost}
Let $(P,w,Q',f)$ be a query space over $\mathbb{P}$. Then, for a query $q\in Q'$ we define the total cost of $q$ as
$f(P,w,q) = \sum_{p\in P} w(p)f(p,q).$
\end{definition}

In the next definition, we describe formally a (strong) coreset for a given optimization problem.
\begin{definition}[Traditional Coresets]\label{def:strongCoreset}
For a query space $(P,w,Q',f)$, and an error parameter $\eps\in (0,1)$, an $\eps$-coreset is a pair $(C,u)$ such that $C\subseteq P$, $u:C\to \REAL$ is a weight function, and for every $q\in Q'$, $f(C,u,q)$ is a $1\pm\eps$ multiplicative approximation for  $f(P,w,q)$, i.e., 
\begin{equation}
\abs{f(P,w,q) - f(C,u,q)} \leq \eps f(P,w,q). \label{eq:traditional-coreset}
\end{equation}
\end{definition}

\section{Method}
In this section we first explain our approach in general, emphasising its novelty and then, we present our suggested framework including all the details.
\subsection{Novel Framework}
We propose a \emph{practical} and \emph{generic} framework for coreset construction to a wide family of problems via the following steps:
\begin{enumerate}[(a)] 
  \item \emph{Make a problem simpler by relaxing the definition of a coreset.}\label{framework:1} Namely, we propose a new $(\eps,\mu)$-coreset for the Average Loss  (in  Definition~\ref{def:aca}) that is a relaxation of the standard definition (in Definition~\ref{def:strongCoreset}), and is more suited for the learning formalism.
  \item \emph{Define coreset construction as a learning problem}.  Here, the coreset (under the new definition in Definition~\ref{def:aca}) is the training variable. \label{framework:2}
  \item \emph{Find the coreset that optimizes the empirical risk over a training set of queries.} \label{framework:3} We assume that we are given a set of queries, chosen i.i.d. from an unknown distribution and we find a coreset that approximates the average loss of the original input data over the training set of queries. 
  \item \emph{Show that the optimized coreset generalizes to all members in the query set.} \label{framework:4} Namely, the expected loss on the coreset over all queries approximates the expected loss on the original input data.
\end{enumerate}

\subsection {$(\eps,\mu)$-Coreset for the Average Loss}
We relax the definition of a coreset by observing that in data mining and machine learning problems we are usually interested in approximating the average loss over the whole set of queries rather than approximating the loss of a specific query. To this end, we define a distribution over the set of queries in Definition~\ref{def:prob-querySpace}, and then focus on approximating the expected loss in Definition~\ref{def:aca}.

\begin{definition} [Measurable query space] \label{def:prob-querySpace}
Let $(P,w,Q',f)$ be a query space over the ground set $\mathbb{P}$, and let $\mu$ be a probability measure on a Probability space
 $(Q',2^{Q'})$. Then, the tuple $(P,w,Q',f,\mu)$ is called a measurable query space over $\mathbb{P}$. 
\end{definition}

\begin{definition} [$(\eps,\mu)$-coreset for the Average Loss]\label{def:aca}
Let $(P,w ,Q',f,\mu)$ be a measurable query space over $\mathbb{P}$. Let $\eps \in [0,\infty)$ be an error parameter, $C\subset \mathbb{P}$ be a set, and $u:C\to\REAL$ be a weight function such that: 

     $$\abs{\mathbb{E}_\mu(f(P,w,q))  -  \mathbb{E}_\mu (f(C,u,q))} \leq\eps,$$ i.e., the expected loss of the original set $P$ over the randomness of sampling a query $q$ from the distribution $\mu$ is approximated by the expected loss on $C$.
   
Then, the pair $(C,u)$ is called an $(\eps,\mu)$-coreset for the measurable query space $(P,w,Q',f,\mu)$.
\end{definition}

While, $(P,w)$ is also an $(\eps,\mu)$-coreset of $(P,w,Q',f,\mu)$, coreset $(C,u)$ is efficient if the cardinality of $C$ is significantly smaller than $P$, i.e., $|C|\ll|P|$, hopefully by order of magnitude.

\textbf{Remark:} Throughout the literature, the term ``coreset'' usually refers to a small weighted \textbf{subset} of the input set (data). However, in other works (and in ours), this requirement is relaxed~\cite{cohen2015dimensionality,phillips2016coresets}. In many applications this relaxation gives a significant benefit as it supports a much larger family of instances as coreset candidates. 

\subsection{Coreset Learning}\label{sec:coreset-learning}

\newcommand{\computeaca}{\textsc{AutoCL}}
\begin{algorithm}[b]
\small
\caption{$\computeaca(P,w,Q,f,\csize)$}
\label{alg:main-new}
\textbf{Input:} {A finite input set $P$, and its weight function $w:P\to \REAL$, a finite set of queries $Q$, a loss function $f:P
\times Q\to [0,\infty)$, and an integer $\csize\geq1$.}
\begin{spacing}{1.1}
\begin{algorithmic}[1]
\small
\STATE $C:=\br{c_i}_{i=1}^{\csize}$ is an arbitrary set of $\csize$  vectors in $\mathbb{P}$.\label{line:initc-new}\\ 
\STATE $u(c):=1/\csize$ for every $c\in C$.\label{line:initu-new}\\
\FOR{$i:= 1 \to \epochs$} \label{lin:opt_start}
    \STATE  $f_C:= \frac{1}{k}\sum_{q\in Q}f(C,u,q)$ \label{line:cerror-new} 
    \COMMENT{The average loss on $C$.} \\ 
    \STATE  $f_P:= \frac{1}{k}\sum_{q\in Q}f(P,w,q)$\label{line:perror-new} \COMMENT{The average loss on $P$.}
    \STATE  $\mathrm{loss}:= \abs{f_P - f_C} + \lambda \abs{\sum_{p\in P}w(p) - \sum_{p\in C}u(p)}$\label{line:apperror-new} \\
    \COMMENT{The approximation error that we wish to minimize, $\lambda>0$ is a hyper-parameter to balance the two losses.}
    \STATE  $C:= C - \alpha \nabla \mathrm{loss}(C)$\label{line:updatec-new} \\ \COMMENT{Update $C$, where $\alpha>0$ is the learning rate.}\\
    \STATE  $u:= \max\{0,u - \alpha \nabla loss(u)\}$ \label{line:updateu-new}\COMMENT{Update $u$.} 
\ENDFOR \label{lin:opt_end}
\STATE \textbf{return} $(C,u)$
\end{algorithmic}
\end{spacing}
\end{algorithm}

We propose to learn a coreset (and its weights) as in Definition~\ref{def:aca} using gradient-based methods. 
We assume that we are given a set $P$, its weights $w$ such that $\sum_{p\in P} w(p) =1$,\footnote{We use this assumption for simplicity of the writing. Practically, we can implement it by scaling the input weights to sum to $1$, and formally, all is needed is scaling the sample size of the queries according to the sum of weights.} and a set $Q$ of $|Q|=k$ queries sampled i.i.d. from $Q'$ (according to the measure $\mu$). First, we aim to compute an $(\eps,\mu)$-coreset $(C,u)$ of $(P,w)$ with respect to the finite set of queries $Q$. Formally speaking, $(C,u)$ should satisfy:
\begin{align}
\left|\sum_{q\in Q}\frac{1}{k}f(P,w,q)-\sum_{q\in Q}\frac{1}{k}f(C,u,q)\right| \leq \eps.  \label{ALG:guarantee}  
\end{align}
 
To do so, we can treat $Q$ as our training data and learn coreset  $(C,u)$ of $(P,w)$ with respect to the objective $f$ by minimizing the following loss: 
$$\left|{\frac{1}{k}\sum_{q\in Q}f(P,w,q)-\frac{1}{k}\sum_{q\in Q}f(C,u,q)}\right|.$$

This will guarantee that $(C,u)$ is an $(\eps,\mu)$-coreset for the measurable query space $(P,w,Q,f,\unif)$, where $\unif:Q\to [0,1]$ is the uniform distribution over the finite set $Q$, i.e., $\unif(q)=1/k=1/|Q|$ for every $q\in Q$. 

However, we wish that the constraint in Eq.~\eqref{ALG:guarantee} would hold for the whole set of queries $Q'$ in order to obtain an $(\eps,\mu)$-coreset for our desired (original) measurable query space $(P,w,Q',f,\mu)$. 
To obtain a generalized solution (as we show in 
Section~\ref{sec:generalization}), we need to bound $\sup_{q\in Q'}{f(C,u,q)}$. To do so, we should guarantee that the sum of coreset weights approximates the original sum of weights, i.e:
\begin{align}
\abs{\sum_{p\in P}w(p) - \sum_{p\in C}u(p)} \leq \eps. \label{ALG:guarantee2}  
\end{align}
The motivation behind bounding Eq~\eqref{ALG:guarantee2} is as follows. Recall that $\mathbb{P}$ is the ground set, i.e., $P,C \subset \mathbb{P}$. Let $M=\sup_{q\in Q',p\in \mathbb{P}}\abs{f(p,q)}$, so that enforcing Eq.~\eqref{ALG:guarantee2}, yields for every $q\in Q'$ 
$$f(C,u,q) \leq \sum_{p\in C}u(p)f(p,q) \leq (\sum_{p \in P}w(p)+\eps)M= (1+\eps)M.$$ 
Hence, we ``force'' our coreset to have a bounded loss over the whole query space $\sup_{q\in Q'}f(C,u,q) \leq (1+\eps)M$, furthermore, this bound is proportional to the bound of the loss on the original input $P$, i.e, it is proportional to $$\sup_{q\in Q'}f(P,w,q)\leq \sum_{p\in P} w(p) M =M,$$ and the approximation error $\eps$.

To summarize, we learn an $(\eps,\mu)$-coreset $(C,u)$ of $(P,w)$ with respect to the objective $f$ given a training data (set of queries) $Q$. To enforce the conditions in Eqs.~\eqref{ALG:guarantee} and~\eqref{ALG:guarantee2} to hold with small $\eps$, we minimize the following loss:  
\begin{equation}
\begin{split}
\mathrm{loss}(Q;C,u)&:=\left|\frac{1}{k}\sum_{q\in Q}f(P,w,q)-\frac{1}{k}\sum_{q\in Q}f(C,u,q)\right| \\ 
&\quad + \lambda \left|{\sum_{p\in P}w(p) - \sum_{p\in C}u(p)}\right|. \label{alg:loss}
\end{split}
\end{equation}
Here, $\lambda>0$ is a hyper-parameter to balance the two losses.
The algorithm for coreset learning is summarised in Algorithm~\ref{alg:main-new}.

\subsection{Generalization} \label{sec:generalization}

We start by stating the sufficient guarantees for the $(\eps,\mu)$-coreset (i.e., the sufficient guarantees to obtain a generalized solution): 
\begin{enumerate}[(i)]
    \item \label{firststep} With high probability, the expected loss on the set $P$ over all queries in $Q'$ (i.e., $\mathbb{E}_\mu(f(P,w,q))$) is approximated by the average loss on the same set $P$ over the sampled set $Q$ of $k$ queries, i.e., with high probability
    $$\abs{\frac{1}{k}\sum_{q\in Q} f(P,w,q) - \mathbb{E}_\mu (f(P,w,q))}\leq \eps.$$
    \item  \label{secondstep}  The same should hold for $(C,u)$,  i.e., with high probability
    $$\abs{\frac{1}{k}\sum_{q\in Q} f(C,u,q) - \mathbb{E}_\mu (f(C,u,q))}\leq \eps.$$
\end{enumerate}
Then, by Eq.~\eqref{ALG:guarantee}, we have that $\frac{1}{k}\sum_{q\in Q} f(C,u,q)$ approximates $\frac{1}{k}\sum_{q\in Q} f(P,w,q)$, hence combining~\ref{firststep} and~\ref{secondstep} with Eq.~\eqref{ALG:guarantee}, yields that  $\mathbb{E}_\mu (f(C,u,q))$ approximates $\mathbb{E}_\mu (f(P,w,q))$.

To show that~\ref{firststep} holds, we rely on Hoeffding's inequality as follows.

\begin{claim}[Mean of Losses] \label{hofff}
Let $(P,w,Q',f,\mu)$ be a measurable query space such that $\sum_{p\in P}w(p)=1$, and let $M=\sup_{q\in Q'}\abs{f(P,w,q)}$. 
Let $\eps \in (0,\infty)$ be an approximation error, and let $\delta \in (0,1)$ be a probability of failure.
Let $Q$ be a sample of $k\geq \frac{2M^2\ln(2/\delta)}{\eps^2} $ queries from $Q'$, chosen i.i.d, where each $q\in Q'$ is sampled with probability $\mu(q)$. Then, with probability at least $1-\delta$, 
$$\abs{\frac{1}{k}\sum_{q\in Q} f(P,w,q) - \mathbb{E}_\mu (f(P,w,q))} \leq \eps .$$ 
\end{claim}

This claim states that, with high probability, the average loss on the set $P$ over the i.i.d sampled set $Q$ of $k$ queries approximates the expected loss on the set $P$ over all queries in $Q'$ (i.e., $\mathbb{E}_\mu(f(P,w,q))$). However, the size $k$ of $Q$ should be large enough and proportional to the approximation error $\eps$, the probability of failure $\delta$, and finally, the maximum loss over every $q\in Q'$, i.e., $\sup_{q\in Q'}{f(P,w,q)}$ (see Claim~\ref{hofff}). Now, recall that $\mathbb{P}$ is the ground set, i.e., $P,C \subset \mathbb{P}$, and $M=\sup_{q\in Q',p\in \mathbb{P}}\abs{f(p,q)}$. 
As we formally show in Section~\ref{ProofofClaim1}, since, $\eps$ and $\delta$ are fixed, and since $\sup_{q\in Q'}{f(P,w,q)} = \sup_{q\in Q'}{\sum_{p\in P}w(p)f(p,q)} \leq  \sum_{p\in P}w(p)M =M$, all is needed for Claim~\ref{hofff} to hold, is to sample enough queries (based on the Hoeffding's inequality).

To show that~\ref{secondstep} holds, we can also use the Hoeffding's inequality, but additionally we need to bound $\sup_{q\in Q}f(C,u,q)$. 
This was the reason for  adding the constraint on the sum of weights:
 $\abs{\sum_{p\in P}w(p) - \sum_{p\in C}u(p)}\leq \eps ,$  to obtain $\sup_{q\in Q}f(C,u,q)\leq (1+\eps)M$. Formally,

\newcommand{\ALG}{ALG}

\begin{claim}\label{main:theroem}
Let $(P,w,Q',f,\mu)$ be a measurable query space over $\mathbb{P}$, where $\sum_{p\in P}w(p)=1$, and let $M=\sup_{q\in Q',p\in \mathbb{P}}\abs{f(p,q)}$. 
Let $\eps\in (0,\infty)$ be an approximation error, $\delta \in (0,1)$ be a probability of failure, and let $\csize \geq 1$ be an integer. 
Let $Q$ be a sample of $k \geq \frac{2((1+\eps)M)^2\ln(2/\delta)}{\eps^2}$ queries from $Q'$, chosen i.i.d, where each $q\in Q'$ is sampled with probability $\mu(q)$. 
Let $(C,u)$ be the output of a call to $\computeaca(P,w,Q, f,\csize)$; see Algorithm~\ref{alg:main-new}. If
\begin{enumerate}
    \item $\abs{\sum_{p\in P}w(p) - \sum_{p\in C}u(p)}\leq \eps ,$ and\label{assump1} 
    \item $\abs{\frac{1}{k}\sum_{q\in Q} f(P,w,q) - \frac{1}{k}\sum_{q\in Q} f(C,u,q)} \leq \eps.$ \label{assump}
\end{enumerate}
Then, we obtain that, with probability at least $1-\delta$,
$$\abs{\mathbb{E}_\mu (f(P,w,q)) - \mathbb{E}_\mu (f(C,u,q))} < 3\eps.$$
\end{claim}
\begin{proof}
See proof in Section~\ref{ProofofClaim2} in the appendix.
\end{proof}

\subsection{Bridging the Gap Between Theory and Practice}

We take one more step towards deriving effective, practical coresets and replace the loss in Eq.~\ref{alg:loss} (and Line~\ref{line:apperror-new} in Algorithm~\ref{alg:main-new}) with a formulation that is more similar to the standard coreset definition, namely,
$loss(q;C,u)=\abs{1 - \frac{ f(C,u,q)}{f(P,w,q)}} + \lambda \abs{\sum_{p\in P}w(p) - \sum_{p\in C} u(p)}$
and we minimize this loss on average over the training set of queries $Q$; See Algorithm~\ref{alg:main} in the appendix.

A solution obtained by Algorithm~\ref{alg:main} aims to minimize the average approximation error over every query $q$ in the sampled set $Q$ and thus is very similar to the Definition~\ref{def:strongCoreset} with the modification of average instead of the worst case. This  enables us to obtain a better coreset in practice that approximates the loss of every query $q$ (as the minimization is on the average approximation error over all queries and not only on the difference between the average losses of the coreset and the original data over all queries). 
Our empirical evaluation in Section~\ref{sec:exp} verifies that the coreset obtained by running Algorithm~\ref{alg:main} generalizes to unseen queries, i.e., the average approximation error of the coreset over all queries is small compared to other coreset construction algorithms. Moreover, we show below that the solution obtained by Algorithm~\ref{alg:main} satisfies Definition~\ref{def:aca}.

Let $(C^*,u^*)$ be a solution that minimizes the average loss in Algorithm~\ref{alg:main}. 
We can find a constant $\eps'>0$, such that 
\begin{align}
\frac{1}{k}\sum_{q\in Q} \abs{1-\frac{f(C^*,u^*,q)}{f(P,w,q)}}\leq \eps'.  \label{ALG:guarantee_new}  
\end{align}
For a constant $\eps$ from Definition~\ref{def:aca}, let $M=\sup_{q\in Q}|f(P,w,q)|$, and let $\eps = \eps'M $. By simple derivations (see Section~\ref{proofeq5} in the appendix) we can show that 

\begin{equation} \label{eq:M}
\begin{split}
&\abs{\frac{1}{k}\sum_{q\in Q}f(P,w,q)-\frac{1}{k}\sum_{q\in Q} f(C^*,u^*,q)} 
\leq \eps.
\end{split}
\end{equation}
Hence by Claim~\ref{main:theroem} the solution obtained by Algorithm~\ref{alg:main} generalizes to the whole measurable query space $(P,w,Q',f,\mu)$ and thus it satisfies the definition of  $(\eps,\mu)$-coreset, while simultaneously satisfying Eq.~\eqref{ALG:guarantee_new} which is closely related to the original definition of coresets as in Definition~\ref{def:strongCoreset}.


\section{Experimental Results}\label{sec:exp}
We proposed a unified framework for coreset construction that allows us to use the same algorithm for different problems. We demonstrate this on the examples of training set reduction for linear and logistic regression in Section~\ref{sec:exp_data_reduction} and on the examples of model size reduction a.k.a. model compression of MLP and CNN in Section~\ref{sec:exp_model_compr}. We show that in both cases our unified framework yields comparable or even better results than previous coresets,  which are specifically fitted to the problem at hand. 

\begin{figure*}
     \centering
     \includegraphics[width=0.3\textwidth,height=0.2\textwidth]{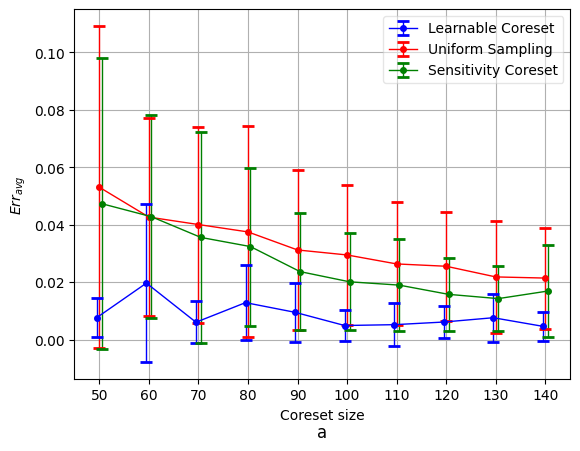}
     \includegraphics[width=0.48\textwidth,height=0.2\textwidth]{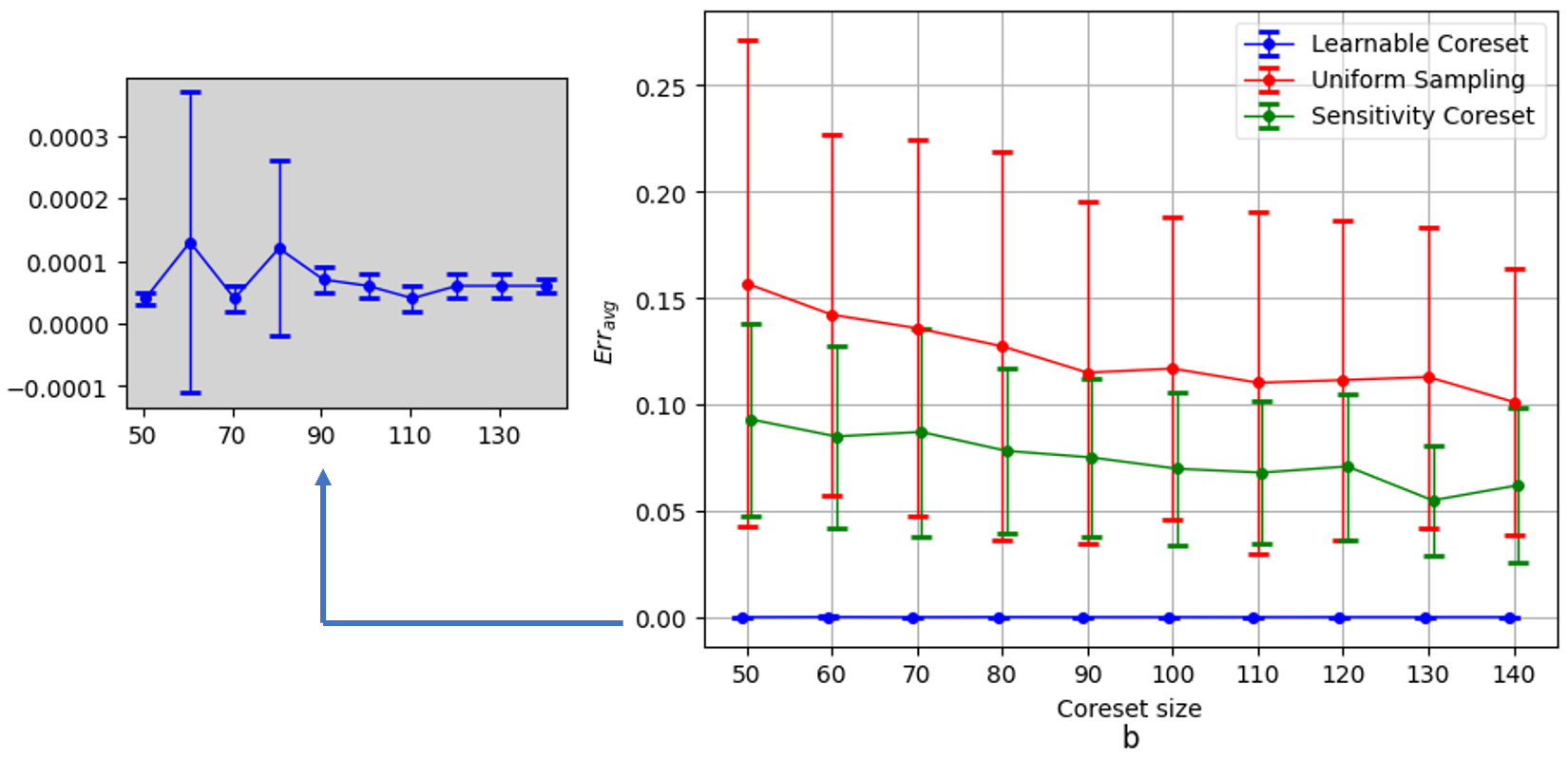}
      \caption{\small Linear regression: a -- Approximation error for the optimal solution as a function of the coreset's size; b -- average approximation error on the unseen test data as a function of the coreset's size.} \label{fig:linearresst}
\end{figure*}

\begin{figure*}
         \centering
         \includegraphics[width=0.3\textwidth,height=0.2\textwidth]{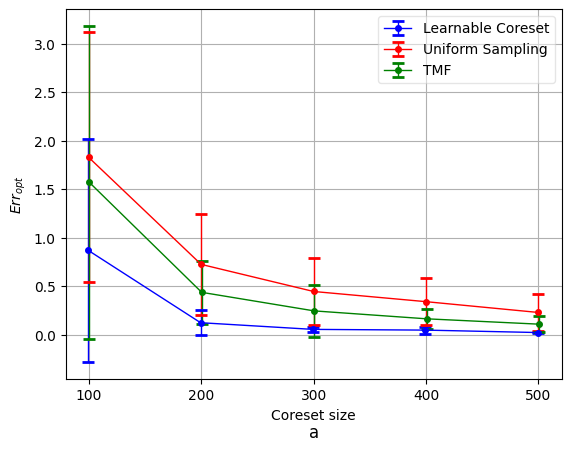}
         \includegraphics[width=0.48\textwidth,height=0.2\textwidth]{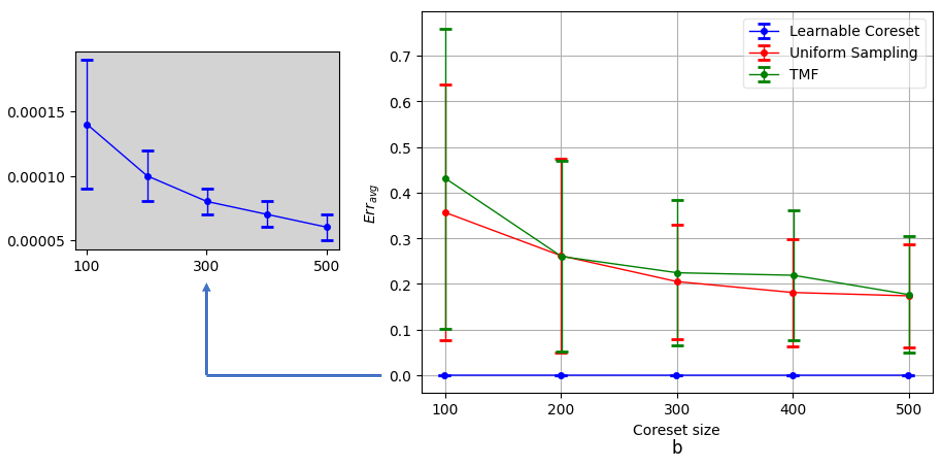}
         \caption{\small Logistic regression: a -- Approximation error for the optimal solution as a function of the coreset's size; b -- average approximation error on the unseen test data as a function of the coreset's size.}
         \label{fig:logisticresst}
\end{figure*}

\subsection{Training Data Coresets}\label{sec:exp_data_reduction}
We demonstrate the practical strength of our coreset construction scheme in the context of data reduction for linear  and logistic regression. 

\subsubsection{Setup} For linear regression, we ran our experiments on the 3D Road Networks dataset\footnote{\url{https://archive.ics.uci.edu/ml/datasets/3D+Road+Network+(North+Jutland,+Denmark)}} (North Jutland, Denmark)~\cite{kaul2013building} that contains 434,874 records. We used  two attributes: ``Longitude'' [Double] and ``Latitude'' [Double] to predict the third attribute ``Height in meters'' [Double]. We created a set of queries by sampling models from training trajectories of linear regression computed using the full data set from 20 random starting points. We split the sampled models into training, validation and tests sets of sizes 20,000 ($|Q|$=20,000), 2,000, 2,000 correspondingly. 
We computed weighted coresets of different sizes, from 50 to 140. For each coreset size, we invoked Algorithm~\ref{alg:main} with Adam optimizer~\cite{kingma2014adam} for 10 epochs with a batch size of 25 and learning rate of 0.01. The results were averaged across $10$ trials. In this experiments we used $\lambda = 1$.

For the logistic regression we performed the experiments on HTRU~\footnote{\url{https://archive.ics.uci.edu/ml/datasets/HTRU2}} dataset, comprising 17,898 radio emissions of the pulsar star represented by 8 features and a binary label~\cite{lyon2016fifty}.  We created a set of queries  similarly to linear regression and we sampled from this set training, validation and test sets of sizes 8,000, 1,600, 800 correspondingly. The results were averaged across $5$ trials.

To make the optimization simpler, we removed the weight fitting term from the loss in Algorithm~\ref{alg:main} and assumed that all members of the coreset have the same weight  $1/|C|$. We ran the optimization for $1000$ epochs with the batch size of $100$ using Adam optimizer and learning rate of $0.001$. Using this modification, we computed coresets of different sizes randing from 100 to 500.

The differences in hyper-parameters and the coreset sizes between the logistic and linear regression experiments are due to the higher complexity of the problem for logistic regression. First, computing a coreset for logistic regression is known to be a complex problem where (high) lower bounds on the coreset size exists~\cite{munteanu2018coresets}. The second (and probably less significant) reason is the dimension of the input data, where we used a higher dimensional input in logistic regression.

\subsubsection{Results} 
We refer to a weighted labeled input dataset by $(P,w,b)$, where $P$ is the dataset, $b:P\to \REAL$ and $w:P\to[0,\infty)$ are the labeling function and weight function respectively, i.e., each point $p$ in $P$ is a sample in the dataset, $b(p)$ is its corresponding label and, $w(p)$ is its weight.
Similarly, we refer to the compressed labeled data set (coreset) by $(C,u,y)$. We report the results using two measures as explained below.
\begin{enumerate}
    \item \textbf{Approximation error for the optimal solution.} Let $q^*$ be the query that minimizes the corresponding objective loss function, e.g., in linear regression: $q^*\in \argmina_{q\in \REAL^{d}} f(P,w,b,q)$, where $f(P,w,b,q) =   \sum_{p\in P}w(p)({p^Tq-b(p)})^2$.
    For each coreset $(C,u,y)$, we compute $q^*_{c}\in \argmina_{q\in \REAL^{d}} f(C,u,y,q)$, then we calculate the approximation error for the optimal solution as 
    $Err_{opt}=\abs{1-\frac{f(P,w,b,q^*_{c})}{f(P,w,b,q^*)}}.$
    
    \item \textbf{Average approximation error.} For every coreset $(C,u,y)$, we report the average case approximation error over every query $q$ in the test set $Q_{test}$, i.e., $Err_{avg}=\frac{1}{|Q_{test}|}\sum_{q\in Q_{test} }\abs{1-\frac{f(C,u,y,q)}{f(P,w,b,q)}}.$
\end{enumerate}

We compare our coresets for linear regression with uniform sampling and with the coreset from~\cite{maalouf2020tight}; $Err_{opt}$ of the three methods is shown in Figure~\ref{fig:linearresst}(a) and $Err_{avg}$ in Figure~\ref{fig:linearresst}(b). 
We compare our coreset for logistic regression with uniform sampling and with the coreset from~\cite{tukan2020coresets}; $Err_{opt}$ of the compared methods is shown in Figure~\ref{fig:logisticresst}(a) and $Err_{avg}$ in Figure~\ref{fig:logisticresst}(b). 

In both experiments we observe that our learned coresets outperform the uniform sampling, and the theoretical counterparts.  Our method yields very low average approximation error, because it was explicitly trained to derive a coreset that minimizes the average approximation error  on the training set of queries, and the learned coreset succeeded to generalize to unseen queries.

\subsection{Model Coreset for Structured Pruning}\label{sec:exp_model_compr}
The goal of model compression is reducing the run time and  the memory requirements during inference with no or little accuracy loss compared to the original model. 
Structured pruning reduces the size of a large trained deep network by reducing the width of the layers (pruning neurons in fully-connected layers and filters in convolutional layers).  An alternative approach is sparsification, which zeros out unimportant parameters in a deep network. The main drawback of sparsification is that it leads to an irregular network structure, which needs a special treatment to deal with sparse representations, making it hard to achieve actual computational savings. Structured pruning simply reduces the size of the tensors, which allows running the resulting network without any amendment. 
Due to the advantage of structured pruning over sparsification, we perform structured pruning of a  deep networks in our experiments.

We assume that the target small architecture is given, and our task is to compute the training parameters of the small architecture that best approximate the original large network. 
We view filters in CNN  or neurons in a fully connected network as items in the full set $P$, and the training data as the query set $Q$. We use the small architecture to define the coreset size in each layer and we learn an equally weighted coreset $C$ (the small network) using Algorithm~\ref{alg:main} and setting $\lambda=0$. 
We report the experiments for structured pruning of a fully connected network in Section~\ref{subsec:neuralpruning} and of channel pruning in Section~\ref{subsec:chanpruning}. 
\subsubsection{Neuron Pruning}\label{subsec:neuralpruning}
\textbf{Setup.} We used LeNet-$300$-$100$ model with 266,610 parameters trained on MNIST~\cite{lecun1998gradient} as our baseline fully-connected model. It comprises two fully connected hidden layers with $300$ and $100$ neurons correspondingly, each followed with a ReLu activation.  After training the baseline model with Adam optimizer for $40$ epochs and batch size of $64$, it  achieved test accuracy of $97.93\%$ and loss = $0.0917$. 
The target small architecture included $30$ neurons in the first layer and $100$ in the second, resulting in  $89.63\%$ compression ratio. We applied the training procedure in Algorithm~\ref{alg:main} to learn the weights of this network using Adam optimizer with $L_2$ regularization for $400$ epochs with the batch size of $500$. 

\noindent\textbf{Results.} The coreset (compressed) model achieved $97.97\%$ accuracy and $0.0911$ loss on the test data, i.e., improvement in both terms. Next, we compare our results to a pair of other coreset-based compression methods in Table~\ref{table:comparison-lenet}, and to non-coreset methods: Filter Thresholding (FT)~\cite{li2016pruning}, SoftNet~\cite{he2018soft}, and ThiNet~\cite{luo2017thinet} implemented in~\cite{Liebenwein2020Provable}. We observe that the learned coreset performs better than most compared methods and comparably to the algorithm derived from the theoretical coreset framework. Note that previous coreset methods~\cite{MussayOBZF20,Liebenwein2020Provable} are designed for a single layer, while our algorithm does not have this limitation and can be applied to compress all layers of the network in a single run. 
Moreover applied to DNN compression, our framework can work on individual weights (sparcification), neurons (as shown above) and channels (as we show next).
\subsubsection{Channel Pruning}\label{subsec:chanpruning}
\textbf{Setup.} We used Pytorch implementation of VGGNet-19 network \footnote{\href{https://github.com/Eric-mingjie/network-slimming/blob/master/models/vgg.py}{VGG-code-link}}
for CIFAR10 from~\cite{liu2017learning} with about 20M parameters as our baseline CNN model (see Table~\ref{table:vggarch} for more details). 
The baseline accuracy and loss in our experiments was  $93.25\%$ and $0.3387$ correspondingly. 
The target architecture\footnote{https://github.com/foolwood/pytorch-slimming} of the small network (see Table~\ref{table:vggarch})
corresponds to 70\% compression ratio and to the reduction of the parameters by roughly 88\%. We ran Algrothm~\ref{alg:main} using the small architecture to define the size of each layer for $180$ epochs with batch size of $500$ using Adam optimizer and $L_2$ regularization. 

\noindent\textbf{Results.} Our compressed model improved the baseline network and achieved $93.51\%$ accuracy and $0.32$ loss. 
Table~\ref{table:comparison-vgg} compares the small network accuracy of the learned coreset with the channel pruning coreset from~\cite{Mussai20a} and several non-coreset methods. While the results are comparable, our algorithm is much simpler and is not tailored to the problem at hand. The coreset reported in~\cite{Mussai20a} was constructed by applying a channel pruning coreset in a layer by layer fashion, while our learned coreset is computed in one-shot for the entire network. Finally, we remind the reader that our framework is generic and could be applied to many other applications in addition to compressing DNNs.

\begin{table}[!h]
\begin{center}
 \begin{tabular}{| c || c | c |} 
 \hline
 Layer & {Width (original)} & {Width (compressed)} \\
 \hline
 \hline
 1 & 64 & 49 \\
 \hline
 2 & 64 & 64 \\
 \hline
 3 & 128 & 128\\ 
 \hline
 4 & 128 & 128\\ 
 \hline
 5 & 256 & 256\\ 
 \hline
 6 & 256 & 254\\ 
 \hline 
 7 & 256 & 234\\ 
 \hline
 8 & 256 & 198\\ 
 \hline
 9 & 512 & 114 \\ 
 \hline
 10 & 512 & 41\\ 
 \hline
 11 & 512 & 24\\ 
 \hline
 12 & 512 & 11\\ 
 \hline
 13 & 512 & 14\\ 
 \hline
 14 & 512 & 13\\ 
 \hline
 15 & 512 & 19\\ 
 \hline
 16 & 512 & 104\\ 
 \hline
\end{tabular}
\end{center}
\caption{VGG-19 original and compressed architectures.} \label{table:vggarch}
\end{table}

\begin{table}[!h]
\centering
\begin{adjustbox}{width=1\columnwidth}
\begin{tabular}{|l|ccc|}
\hline
Pruning Method & Baseline & Small Model& Compression \\
& Error(\%)&Error(\%)&Ratio \\\hline \hline
FT\cite{li2016pruning} & 1.59 & +0.35 & 81.68\%\\ \hline
SoftNet~\cite{he2018soft}& 1.59& +0.41& 81.69\% \\ \hline
ThiNet~\cite{luo2017thinet}&1.59& +10.58& 75.01\% \\ \hline
Sample-based& & &\\
Coreset~\cite{Liebenwein2020Provable}&  1.59&+0.41&84.32\% \\ \hline
Pruning& & & \\
via Coresets~\cite{Mussai20a} &2.16 &-0.13& $\sim 90$\%\\ \hline
Learned Coreset (ours)& 2.07 &-0.04 &89.63\%\\
\hline
\end{tabular}
\end{adjustbox}
\caption{\small Neural Pruning of LeNet-300-100 for MNIST. The results of FT, SoftNet, ThiNet and Sample-Based Coreset are reported in~\cite{Liebenwein2020Provable}. `+' and `-' correspond to increase  and decrease in error, respectively.}
\label{table:comparison-lenet}
\end{table}

\begin{table}[!h]
\centering
\begin{adjustbox}{width=1\columnwidth}
\begin{tabular}{|l|ccc|}
\hline
Pruning Method & Baseline & Small Model& Compression \\
& Error(\%)&Error(\%)&Ratio \\\hline \hline
Unstructured
Pruning~\cite{han2015learning}& 6.5 & -0.02 & 80\%\\ \hline
Structured 
Pruning~\cite{liu2017learning}& 6.33 & -0.13 & 70\% \\ \hline
Pruning
via Coresets~\cite{Mussai20a} &6.33 &-0.29 & 70\%\\ \hline
Learned Coreset (ours) & 6.75&-0.26&70\%\\
\hline
\end{tabular}
\end{adjustbox}
\caption{\small Channel Pruning of VGG-19 for CIFAR-10}
\label{table:comparison-vgg}
\end{table}
\section{Conclusions}
We proposed a novel unified framework for coreset learning that is theoretically motivated and can address problems for which obtaining theoretical worst-case guarantees is impossible. Following this framework, we suggested a relaxation of the coreset definition from the worst case to the average loss approximation. 
We proposed a learning algorithm that inputs a sample set of queries and a loss function associated with the problem at hand and outputs an average-loss coreset that holds for the training set of queries and  generalizes to unseen queries. 
We showed that if the sample set of queries is sufficiently large, then the average loss over the coreset closely approximates the average loss over the full set for the entire query space. 
We then showed empirically, that our learned coresets are capable to generalize to unseen queries even for arbitrary sampling sizes. Our experiments demonstrated that coresets learned by our new approach yielded comparable and even better approximation of the optimal solution loss and average loss over the unseen queries than coresets that have worst-case guarantees. Moreover, our method applied to the problem of deep networks pruning provides the first full-network coreset with excellent performance.
In future work we will try reducing the sampling bound and will apply the proposed framework to derive new coresets.

\bibliographystyle{plain}
\bibliography{main}

\appendices

\section{Hoeffding Theorem}
\begin{theorem}[Hoeffding]\label{thm:hof}
Let $X_1,\cdots,X_k$ be independent random variables, when it is known that for every $i\in k$, $X_i$ is strictly bounded by the intervals $[a_i, b_i]$. Define the empirical mean of these variables by
$\mu = \sum_{i=1}^k \frac{1}{k} X_i$, then
$$\Pr(\abs{\mu - E(\mu)} \geq \eps) \leq 2 e^{-\frac{2k^2\eps^2}{\sum_{i=1}^{k}(a_i-b_i)^2}}$$
\end{theorem}

\section{Proof of Claim~\ref{hofff}} \label{ProofofClaim1}
\begin{proof}
First, observe that: (i) the probability distribution $\mu$ is defined over the set $Q'$, and (ii) the function $f$ in our case is a function of $q \in Q'$, since $P$ and $w$ are fixed (given). Thus, we can define the corresponding probability distribution $\mu'$ for the (multi)-set $F'=\br{f(P,w,q) \mid q\in Q'}$ as follows: For every $x=f(P,w,q) \in F'$ (where $q\in Q'$) we have that $\mu'(x) = \mu'(f(P,w,q)) = \mu (q)$. 

Moreover, the sampled set $Q$ has its corresponding sampled losses set $F=\br{f(P,w,q) \mid q\in Q}$. Hence, we have that
\begin{align}
&\Pr\left(\abs{\frac{1}{k}\sum_{q\in Q} f(P,w,q) - \mathbb{E}_\mu (f(P,w,q))} < \eps\right) 
\\&=1 - \Pr\left(\abs{\frac{1}{k}\sum_{q\in Q} f(P,w,q) - \mathbb{E}_\mu (f(P,w,q))} \geq \eps\right).  \label{step:1}
\end{align} 
By applying Hoeffding’s inequality (see Theorem~\ref{thm:hof} in the appendix) we have: 
\begin{equation}
\begin{split}
&1 - \Pr\left(\abs{\frac{1}{k}\sum_{q\in Q} f(P,w,q) - \mathbb{E}_\mu (f(P,w,q))} \geq \eps\right)
\\ &\geq 1 - 2e^{-2 \frac{|F|^2 \cdot  {\eps}^2} {\sum_{i = 1}^{|F|} (a_i -b_i)^2}},\label{step:2}
\end{split}
\end{equation}
where $a_i$ and $b_i$ are the lower and upper bounds on the loss of the $i$th sampled query respectively. 

Since, by the definition of $M$ we have that for every $q\in Q'$, $-M\leq f(P,w,q)\leq M$, we obtain that,

\begin{equation}
\begin{split}
1 - 2e^{-2 \frac{|F|^2 \cdot  {\eps}^2} {\sum_{i = 1}^{|F|} (a_i -b_i)^2}}
&= 1 - 2e^{-2 \frac{|F|^2 \cdot  {\eps}^2} {\sum_{i = 1}^{|F|} (M +M)^2}}
\\ &= 1 - 2e^{-2 \frac{|F|^2 \cdot  {\eps}^2} {|F|(2M)^2}} 
=1 - 2e^{-2 |F| \cdot  {\eps}^2/4M^2}.\label{plug:a_ib_i}
\end{split}
\end{equation}

Pluging $|F|=|Q| \geq \frac{2M^2\ln(2/\delta)}{\eps^2}$ in~\eqref{plug:a_ib_i} yields
\begin{equation}
\begin{split}
&1 - 2e^{-2 |F| \cdot  {\eps}^2/4M^2} 
\geq 1 - 2e^{-2 \frac{2M^2\ln(2/\delta)}{\eps^24M^2} \cdot  {\eps}^2} 
\\&= 1 - 2e^{ -{\ln(2/\delta)}} = 
1 - 2\cdot \frac{\delta}{2} = 1-\delta.\label{step:3}
\end{split}
\end{equation}
Finally, combining~\eqref{step:1},~\eqref{step:2}, and~\eqref{step:3} proves the claim as
$$ \Pr\left(\abs{\frac{1}{k}\sum_{q\in Q} f(P,w,q) - \mathbb{E}_\mu (f(P,w,q))} < \eps\right)  \geq  1-\delta.$$

\end{proof}

\section{Proof of Claim~\ref{main:theroem}} \label{ProofofClaim2}
\begin{proof}
Let $M_1= \sup_{q\in Q'}f(P,w,q)$, and let $M_2 = \sup_{q\in Q'}f(C,u,q)$. First we observe that, 
\begin{equation}\label{M1bound}
\begin{split}
M_1&= \sup_{q\in Q'}f(P,w,q) = \sup_{q\in Q'}\sum_{p\in P}w(p) f(p,q) \\ 
&\leq \sum_{p\in P}w(p) M=M, 
\end{split}
\end{equation}
where the third derivation holds by the definition of $M$, and the fourth holds since $\sum_{p\in P}w(p)=1$. 
We also have,
\begin{equation}\label{M2bound}
\begin{split}
M_2 &= \sup_{q\in Q'}f(C,u,q) = \sup_{q\in Q'}\sum_{p\in C}u(p) f(p,q) \\
&\leq \sum_{p\in C}u(p) M \leq (\eps + \sum_{p\in P}w(p)) M=(1+\eps)M, 
\end{split}
\end{equation}
where the third inequality holds by the definition of $M$, the fourth by Assumption~\ref{assump1}, and the last holds since $\sum_{p\in P}w(p)=1$. By combining ~\eqref{M1bound} and~\eqref{M2bound} we get that $(1+\eps)M\geq M_1,M_2$.

Now, we note that Claim~\ref{hofff} holds for any measurable query space. Hence, for the pair of measurable query spaces $(P,w,Q',f,\mu)$ and $(C,u,Q',f,\mu)$, if the sampled set $Q\subset Q'$ satisfies that $k=|Q|\geq\frac{2((1+\eps)M)^2\ln(2/\delta)}{\eps^2}$, then by Claim~\ref{hofff} we get that:
\begin{align}
\abs{\frac{1}{k}\sum_{q\in Q} f(P,w,q) - \mathbb{E}_\mu (f(P,w,q))} < \eps ,\label{eq:P} 
\end{align}
and
\begin{align}
\abs{\frac{1}{k}\sum_{q\in Q} f(C,u,q) - \mathbb{E}_\mu (f(C,u,q))} < \eps. \label{eq:C} 
\end{align}

By the triangle inequality we have that
\begin{align}
&\abs{\mathbb{E}_\mu (f(P,w,q)) - \mathbb{E}_\mu (f(C,u,q))}  \nonumber
\\ &\leq \abs{\mathbb{E}_\mu (f(P,w,q))  - \frac{1}{k}\sum_{q\in Q} f(P,w,q)}\label{step:bound1}
\\ &+ \abs{\frac{1}{k}\sum_{q\in Q} f(P,w,q) - \frac{1}{k}\sum_{q\in Q}f(C,u,q)} \label{step:bound2}
\\ &+ \abs{\frac{1}{k}\sum_{q\in Q} f(C,u,q) - \mathbb{E}_\mu(f(C,u,q))}.\label{step:bound3}
\end{align}
By~\eqref{eq:P} and~\eqref{eq:C}, and by the assumption~\eqref{assump} on the output $C,u$, we have that~\eqref{step:bound1},~\eqref{step:bound2}, and~\eqref{step:bound3} are bounded by $\eps$. Hence, 
$$\abs{\mathbb{E}_\mu (f(P,w,q)) - \mathbb{E}_\mu (f(C,u,q))}  \leq 3\eps.$$

\end{proof}

\section{Practical implementation}\label{practicimp}

\newcommand{\computeacaPractical}{\textsc{Practical-AutoCL}}

\begin{algorithm}[htb]
\small
\caption{$\computeacaPractical(P,w,Q,f,\csize)$}
\label{alg:main}
\textbf{Input:} A finite input set $P$ and its weight function $w$, a finite set of queries $Q$, a loss function $f:P\times Q\to [0,\infty)$, and an integer $\csize\geq1$. \\
\begin{spacing}{1.1}
\begin{algorithmic}[1]
\small
\STATE $C:=\br{c_i}_{i=1}^{\csize}$ is an arbitrary set of $\csize$  vectors in $\mathbb{P}$.\label{line:initc}\\  
\STATE $u(c):=1/\csize$ for every $c\in C$.\label{line:initu}
\FOR{$i\in \br{1,\cdots, \epochs}$} 
    \FOR{every $q\in Q$}
      \STATE $f_C:= f(C,u,q)$ \label{line:cerror} \\\COMMENT{The cost of the query $q$ on $C$.} 
      \STATE$f_P:= f(P,w,q)$\label{line:perror} \\ \COMMENT{The cost of the query $q$ on $P$.}
      \STATE$loss:= \abs{1-\frac{f_C}{f_P}}  + \lambda\abs{\sum_{p\in P}w(p) - \sum_{p\in C}u(p)}$\label{line:apperror} \\ \COMMENT{The approximation error that we wish to minimize}\\
      \COMMENT{$\alpha$ is the learning rate.}
      \STATE$C:= C - \alpha \nabla loss(C)$\label{line:updatec} \COMMENT{Update $C$} 
      \STATE$u:= \max\{0,u - \alpha \nabla loss(u)\}$ \label{line:updateu}\COMMENT{Update $u$.} 
      \ENDFOR 
\ENDFOR
\STATE \textbf{return} $(C,u)$
\end{algorithmic}
\end{spacing}
\end{algorithm}
While the training of Algorithm~\ref{alg:main} is formalized as a stochastic process, i.e., sequentially, for every $q\in Q$, we compute the approximation error $\abs{1 - \frac{ f(C,u,q)}{f(P,w,q)}}$ for this one query $q$, we then update the learned variables based on this error. 
However, it can be implemented using a minibatch of several queries $\tilde{Q}\subseteq Q$. Here, the approximation error with respect to the current batch $\tilde{Q}$ is $\sum_{q\in \tilde{Q}} \abs{1 - \frac{ f(C,u,q)}{f(P,w,q)}}$ and the learned variables are updated based on this error.
\subsection{Proof of Equation~\ref{eq:M}} \label{proofeq5}
For a constant $\eps$ from Definition~\ref{def:aca}, let $M=\sup_{q\in Q}|f(P,w,q)|$, and let $\eps = \eps'M $. We show that
\begin{equation} \label{eq:M1}
\begin{split}
&\abs{\frac{1}{k}\sum_{q\in Q}f(P,w,q)-\frac{1}{k}\sum_{q\in Q} f(C^*,u^*,q)} 
\leq \eps,
\end{split}
\end{equation}
\begin{proof}

\begin{equation} \label{eq:Mappend}
\begin{split}
&\abs{\frac{1}{k}\sum_{q\in Q}f(P,w,q)-\frac{1}{k}\sum_{q\in Q} f(C^*,u^*,q)} 
\\&\leq \frac{1}{k}\sum_{q\in Q}\abs{f(P,w,q) - f(C^*,u^*,q)}
\\ &=  \frac{M}{Mk}\sum_{q\in Q} \abs{f(P,w,q)- f(C^*,u^*,q)}
\\&\leq  \frac{M}{k}\sum_{q\in Q} \frac{1}{|f(P,w,q)|} \abs{\sum_{q\in Q}f(P,w,q) -f(C^*,u^*,q)}
\\&= \frac{M}{k}\sum_{q\in Q} \abs{1-\frac{f(C^*,u^*,q)}{f(P,w,q)}} 
\leq M\eps' =\eps,
\end{split}
\end{equation}
where the first derivation holds since $|\sum_{i=1}^k{a_i}|\leq \sum_{i=1}^k\abs{a_i}$  for any set of number $\br{a_i}_{i=1}^k$, the derivation in (\ref{eq:M}) follows from the definition of $M$, i.e., since $M> |f(P,w,q)|$ for every $q\in Q$, the one after holds since $|a||b|=|ab|$ for any pair $a,b\in \REAL$, and the last derivation holds by~\eqref{ALG:guarantee_new}.
\end{proof}




\ifCLASSOPTIONcaptionsoff
  \newpage
\fi

\begin{IEEEbiographynophoto}{Alaa Maalouf}
received his B.Sc. and M.Sc. in Computer Science at the University of Haifa, Israel, in 2016 and 2019 respectively, and is now a Ph.D. student under the supervision of Prof. Dan Feldman. His main research interests focus on Machine/Deep Learning, Robotics, Computational Geometry and Coresets (data summarization) for Big Data. 
\end{IEEEbiographynophoto}

\begin{IEEEbiographynophoto}{Gilad Eini}
received his B.Sc. in Computer Science at the University of Haifa, Israel, in 2017, and is on the verge of finishing his  M.Sc. under the supervision of Prof. Dan Feldman. His main research interests focus on Machine/Deep Learning, Computer vision and Coresets (data summarization) for Big Data.
\end{IEEEbiographynophoto}

\begin{IEEEbiographynophoto}{Ben Mussay}
Ben Mussay received the BSc degree and the MSc degree in computer science from the Univeristy of Haifa, Israel,
in 2019 and 2020, respectively. His reseach interests are sublinear algorithms and deep learning.
\end{IEEEbiographynophoto}

\begin{IEEEbiographynophoto}{Dan Feldman}
is an associate professor and the head of the Robotics and Big Data Lab at the University of Haifa, after returning from a 3 years post-doc at at Caltech and MIT.
During his PhD at the University of Tel-Aviv he developed data reduction techniques known as core-sets, based on computational geometry. Since his post-docs, Dan's coresets are applied for main problems in Machine Learning, Big Data, computer vision, EEG and robotics. His group in Haifa continues to design and implement core-sets with provable guarantees for such real-time systems.
\end{IEEEbiographynophoto}

\begin{IEEEbiographynophoto}{Margarita Osadchy}
Margarita Osadchy is an Associate Professor in the Department of Computer Science at the University of Haifa. 
She is a member of the Data Science Research Center and the member of the scientific committee of the
Center for Cyber Law and Policy at the University of Haifa. She received the PhD degree with honors in computer
science from the University of Haifa, Israel. She was a visiting research scientist at the NEC Research
Institute and then a postdoctoral fellow in the Department of Computer Science at the Technion. Her main 
research interests are deep learning, machine learning, computer vision, and computer security and privacy.
\end{IEEEbiographynophoto}




\end{document}

%% file: preamble.tex
\usepackage[utf8]{inputenc} 
\usepackage[T1]{fontenc}    
\usepackage{url}            

\usepackage{algorithm}
\usepackage{algorithmic}

\usepackage[export]{adjustbox}
\usepackage{subcaption}
\usepackage[percent]{overpic}
\usepackage{wrapfig}
\usepackage{graphicx}
\usepackage{svg}

\usepackage{pifont}
\usepackage{float}

\usepackage{amsmath,amsfonts,amsthm,amssymb, nccmath}
\usepackage{nicefrac}       
\allowdisplaybreaks 
\usepackage{bbm}
\usepackage{framed}
\usepackage{enumerate}
\usepackage{mathtools}
\usepackage{scalefnt}
\usepackage{changepage}

\usepackage{comment}
\usepackage{setspace}
\usepackage{verbatim}
\usepackage{xcolor}
\usepackage{scalefnt}
\usepackage{xr}

\usepackage[shortlabels]{enumitem}
\setlist{nolistsep, noitemsep,leftmargin=12pt, labelwidth=1pt, topsep=0pt, partopsep=0pt}

\usepackage{placeins}
\usepackage{color}
\usepackage{footmisc}
\usepackage{xspace}

\usepackage{thmtools}
\usepackage{thm-restate}

\usepackage{multirow}
\usepackage{booktabs}
\usepackage{array}
\usepackage{tabularx}

\usepackage{microtype}

\usepackage{hyperref}

\DeclareMathAlphabet{\pazocal}{OMS}{zplm}{m}{n}
\newcommand{\unif}{\pazocal{U}}

%% file: commands.tex
\makeatletter
\newcommand{\settitle}{\@maketitle}
\makeatother

\DeclarePairedDelimiterX{\dotp}[2]{\langle}{\rangle}{#1, #2}

\makeatletter
\newcommand*\bigcdot{\mathpalette\bigcdot@{.5}}
\newcommand*\bigcdot@[2]{\mathbin{\vcenter{\hbox{\scalebox{#2}{$\m@th#1\bullet$}}}}}
\makeatother

\newcommand{\be}{\begin{equation}}
\newcommand{\ee}{\end{equation}}

\newcommand{\br}[1]{\left\{#1\right\}}

\newcommand{\abs}[1]        {| #1 |}

\newcommand{\eps}{\ensuremath{\varepsilon}}                       
\renewcommand{\epsilon}{\varepsilon}

\declaretheorem[name=Theorem]{theorem}

\declaretheorem[name=Definition]{definition}

\newcommand{\REAL}{\ensuremath{\mathbb{R}}}

\let\Pr\relax
\DeclareMathOperator*{\Pr}{\mathbf{Pr}}

 \newtheorem{claim}[theorem]{Claim}
 
\newcommand{\csize}{C_{size}}

\newcommand{\epochs}{epochs}
\DeclareMathOperator*{\argmina}{arg\,min}